\documentclass{amsart}
\usepackage{graphicx}
\usepackage{geometry}
\usepackage[colorlinks, linkcolor=blue,  citecolor=blue, urlcolor=blue]{hyperref}

\DeclareMathOperator{\softmax}{{softmax}}

\DeclareMathOperator*{\argmax}{{arg\,max}}

\newcommand{\R}{\mathbb{R}}

\newcommand{\Rn}{\R^n}

\newcommand{\loss}{\mathcal{L}}

\newcommand{\LD}{L_{\mathcal D}}

\newcommand{\fM}{f}

\newcommand{\PP}{\rho} 
\newcommand{\DataSupp}{\mathcal{X}}
\newcommand{\DS}{S}

\newcommand{\YY}{\mathcal{Y}}

\newcommand{\EE}{\mathbb{E}}

\usepackage{mathtools}
\mathtoolsset{showonlyrefs}
\usepackage{enumerate}
\graphicspath{{figs2/}}

\newtheorem{theorem}{Theorem}[section]
\newtheorem{lemma}[theorem]{Lemma}

\theoremstyle{definition}

\newtheorem{example}[theorem]{Example}

\theoremstyle{remark}
\newtheorem{remark}[theorem]{Remark}

\numberwithin{equation}{section}

\begin{document}

\title[Regularization for machine learning]{Partial differential equation regularization  for supervised machine learning}


\author{Adam M. Oberman}
\email{adam.oberman@mcgill.ca}
\thanks{This material is based on work supported by the Air Force Office of Scientific Research under award number FA9550-18-1-0167}
%


\date{\today}


\begin{abstract}
	This article is an overview of supervised machine learning problems for regression and classification.  Topics include: kernel methods, training by stochastic gradient descent, deep learning architecture, losses for classification, statistical learning theory, and dimension independent generalization bounds. Implicit regularization in deep learning examples are presented, including data augmentation, adversarial training, and additive noise.  These methods are reframed as explicit gradient regularization.   
\end{abstract}

\maketitle


\section{Introduction}
In this work we present a mathematically oriented introduction to the gradient regularization approach to deep learning, with a focus on convolutional neural networks (CNNs) for image classification.   Image classification by deep neural networks is now well established, as exemplified by the impressive performance of models on  data sets such as those presented in Example~\ref{ex:dataset}.  The goal is to build a classification map (model) from images to labels, as illustrated in Figure~\ref{fig:ImageNetMap}. 

\begin{figure}
  \includegraphics[height = 6cm]{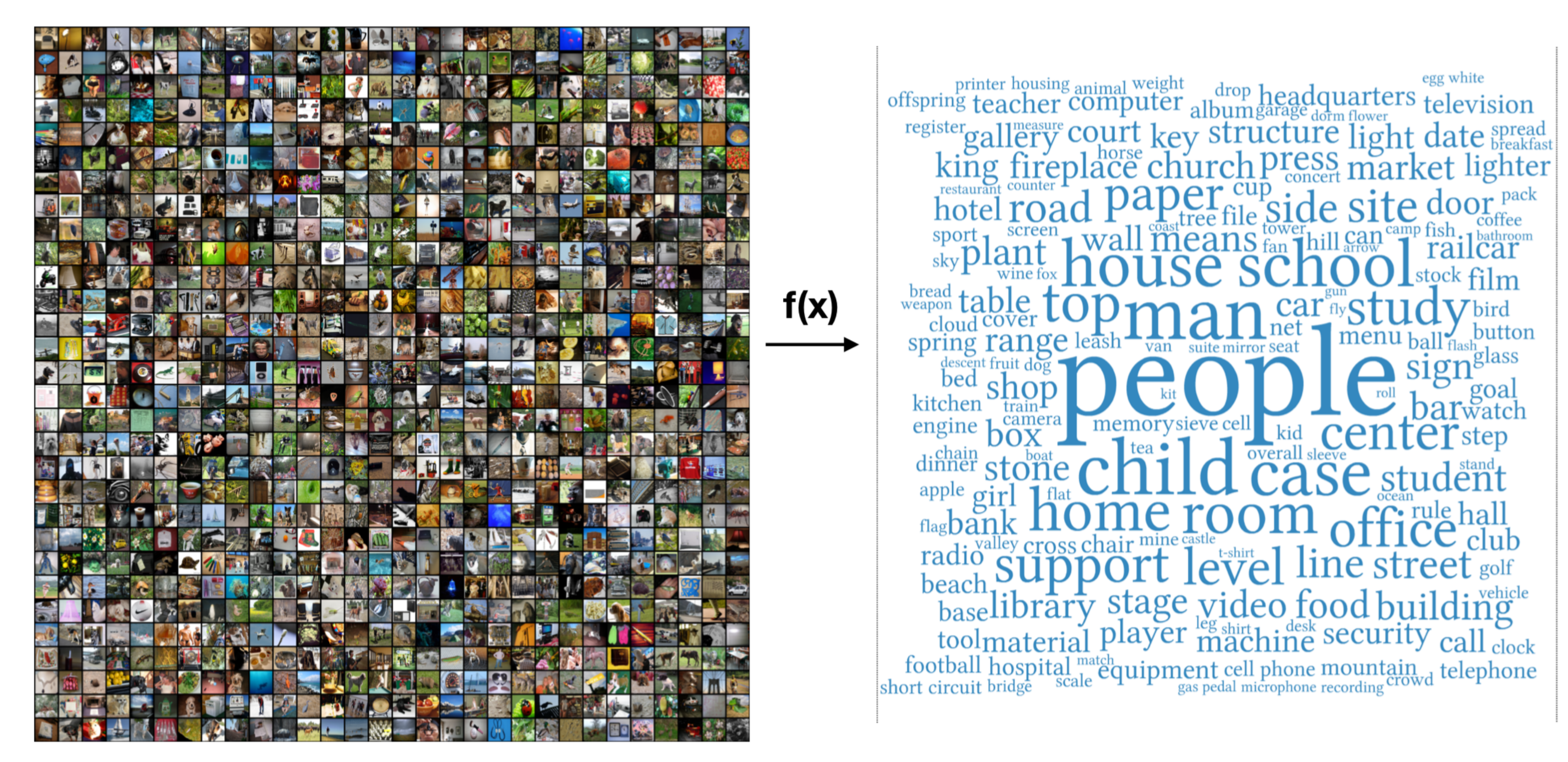}
  \caption{Illustration of the classification map $f(x):\DataSupp\to \YY$ on the ImageNet dataset.}
  \label{fig:ImageNetMap}
\end{figure}

\begin{example}\label{ex:dataset}
The MNIST dataset consists of $m = 70,000$,  $d = 28\times 28$ greyscale images of handwritten digits.  CNNs achieve an error of less than 1\% on MNIST.  On this simple data set, support vector machines (SVM) achieve accuracy almost as high.
The CIFAR-10 dataset consists of $m = 60,000$ $d =32\times 32\times 3 = 3072$ colour images in $K=10$ classes, with 6000 images per class.  CNNs can achieve accuracy better that 96\% on this dataset. 
	ImageNet has $K=21,841$ class labels with $m=14,197,122$ total images, at a resolution of  $256\times 256$ pixels with 3 color channels, so $d = 256\times 256\times 3 = 196608$.   See Figure~\ref{fig:ImageNetMap} for an illustration of the classification map on ImageNet.  Current models achieve accuracy greater than 84\% and Top 5 accuracy (meaning the correct label is in the five highest ranked predictions) better than~97\%. 

\end{example}

 Machine learning can be divided into supervised, unsupervised, and reinforcement learning.  Supervised learning can be further divided into regression, where the target function value is a number, and classification, where the target is labels.  Deep learning, which refers to machine learning using the deep neural network hypothesis class, is having an impact in many areas, far too many to cover effectively here.  By focussing on CNNs for classification and by writing for a mathematical audience, we hope to make a contribution to this literature.   

 Another topic which may be of interest to readers  is generative models using convolutional neural networks~\cite{goodfellow2014generative}. These models sample a distribution with the goal of generating new images from the distribution.  Recent work has exploited connections with Optimal Transportation \cite{1701.07875}.

\section{Machine Learning problem definition}\label{sec:ML}
Let $x \in \DataSupp$ be an input, where $\DataSupp \subset [0,1]^d$ and the dimension is large, $d \gg 1$.  We consider a target set $\YY$, which will be either  $\YY =  \R$, for regression problems, or   $\YY = \{1, \dots, K\}$, for (mono-label) classification problems.  The dataset
\begin{equation}\label{dataset}
	\DS_m = \{(x_1, y_1), \dots , (x_m, y_m) \}
\end{equation}
consists of $m$ samples, $x_i$, drawn i.i.d.\ from a data distribution, $\rho(x)$, with support $\DataSupp$.  The labels $y_i \in \YY$ are also given.  

The non-negative function $\loss:\YY \times \YY \to \R$ is a loss function if it is zero only when $y_1 = y_2$.  
We say the loss is convex if it is convex as a function of $y_1$ for every fixed $y_2$. 
For regression problems, the quadratic loss $\loss(y_1,y_2) = \|y_1 - y_2\|^2$ is often used.  

Our objective is to find a function $f: \DataSupp \to  \YY$ which minimizes the expected loss 
\begin{equation}\label{Loss_Gen}
\LD(f)= \EE_{x\sim \PP }[{\loss }(f(x),y(x))] = \int_{\DataSupp} {\loss }(f(x),y(x)) d\rho(x)
\end{equation}
The expected loss depends on unavailable data and labels, so it needs to be approximated.  
One common practice is to  divide the data set into a training set and a holdout test set  which is not used for training.  Then the expected loss is estimated on the test set.
This procedure allows  for training by minimizing the empirical loss
\begin{equation}
	\tag{EL}\label{ELH}
  L_{\DS_m}[f]= 
	\frac{1}{m} \sum_{ i = 1}^m {\mathcal L}(f(x_i),y_i).
\end{equation}
Typically, the functions considered are restricted to a parametric class
\begin{equation}\label{Hdefn}
	\mathcal H = \{ f(x,w) \mid w \in \R^D \}
\end{equation}
so that minimization of \eqref{ELH} can be rewritten as  the finite dimensional optimization problem
\begin{equation}
	\tag{EL-W} \label{ELW}
	\min_{w}  \frac{1}{m} \sum_{ i = 1}^m {\mathcal L}(f(x_i,w),y_i).
\end{equation}
The problem with using~\eqref{ELW} as a surrogate for \eqref{Loss_Gen}, is that a minimizer of \eqref{ELW} could \emph{overfit}, meaning that the expected loss is much larger than the empirical loss.  Classical machine learning methods avoid overfitting by restricting $\mathcal H$ to be a class of simple (e.g. linear) functions.

\begin{remark}
Some machine learning textbooks focus on a statistical point of view, which is important for problems where is measurement or label noise, or when there is prior statistical information. For example in linear regression, we assume a parametric form for $f(x)$, and the goal is to overcome the noise in the data.  
	In contrast, for image classification on benchmark data sets, such as ImageNet, the images are very clear, and the number of incorrect or ambiguous labels is less than a fraction of a percent, which is still small compared to the 4\% error.  Thus to first approximation, we can assume that the images are free of noise, that all labels are correct, and that there are no ambiguous images. In other words, $y_i = y(x_i)$ for a label function $y(x)$.   It is our opinion that the challenge in learning $y(x)$ comes not from uncertainty, noise, or ambiguity, but rather from the complexity of the functions involved.
\end{remark}

The point of view of this work, expanded upon in \cite{FinlayCalder}, is that while  deep learning models are parametric, the high degree of expressibility of deep neural networks  renders the models  effectively nonparametric.  For such problems, regularization of the loss is necessary for generalization. Regularization is a topic in machine learning as well, but it is interpreted in a wider sense.  We will show that some implicit regularization methods in learning can be reinterpreted as explicit  regularized models.


\subsection{Classification losses}
A method for classification with $K$ classes, $\YY = \{1, \dots, K\}$, is to output a a scoring function for each class, $f = (f_1, \dots, f_K)$ and choose the index of the maximum component as the classification
\begin{equation}
	\label{class}
	C(f(x)) = \argmax_j f_j(x)
\end{equation}
The relevant loss is the 0-1 classification loss,
 \begin{equation}
\loss_{0,1}(f,k) = \begin{cases}
	0 & \text{ if } C(f) = k\\
	1 & \text{ otherwise}
\end{cases}
 \end{equation}
However, this loss is both discontinuous and nonconvex, so a surrogate convex loss is used in practice.   The margin of the function $f(x)$ is given by
\begin{equation}
	\label{max_loss}
\loss_{\max} (f,k) = \max_i f_i - f_k
\end{equation}
which is convex upper bound to the classification loss, making it a convex surrogate loss \cite[Ch 4.7]{mohri2018foundations}.  This loss was proposed in~\cite{crammer2001algorithmic} and studied in~\cite{zhang2004statistical}.  
The non-convex margin loss \cite[Ch 9]{mohri2018foundations} is used to obtain margin bounds for multi-class classification.

\section{Function approximation and regularization}\label{sec:ApproxTheory}

\subsection{Regularization}
The problem of fitting a function can be cast as a multi objective problem: (I) fit the given data points, (II) reduce overfitting.  Minimizing the empirical loss corresponds to a relaxation of (I).  Choosing a a class of functions (hypothesis space) $\mathcal H$  which does not overfit corresponds to a hard constraint on (II) and leads to \eqref{ELH}.  Defining a regularization functional $\mathcal R(f)$ which is some measure of overfit, leads to a soft constraint on (II).  Hypothesis classes are often parametric, but they can also be defined using the regularization functional, as $\mathcal H = \{ f \mid \mathcal R(f) \leq C \}$.   See Table~\ref{table:1}.

The regularized empirical loss minimization problem takes the form
\begin{equation}
	\tag{EL-R} \label{ELR}
	\min_f \frac{1}{m} \sum_{ i = 1}^m {\mathcal L}(\fM(x_i),y_i) + \lambda \mathcal R(f), 
\end{equation}
where $\lambda$ is a parameter which measures the strength of the regularization term.  
\begin{table}[]
\begin{tabular}{|l|l|l|}
\hline
Constraint      & \multicolumn{1}{c|}{Fit data}                                 & \multicolumn{1}{c|}{No overfit} \\ \hline
Hard  & $f(x_i) = y_i$ for all $i$                                      & $f\in \mathcal H$               \\ \hline
Soft   & $\min_f \tfrac{1}{m} \sum_{ i = 1}^m {\mathcal L}(f(x_i),y_i)$ & $\lambda \mathcal R(f)$                  
\vspace{.03cm}
\\ \hline
\end{tabular}
\vspace{1em} 
\caption{Balancing the objectives of fitting the data and not overfitting on unseen data, using hard or soft constraints.}
\label{table:1}
\end{table}

\begin{example}[Cubic splines] 
One dimensional cubic splines are piecewise cubic, twice-differentiable interpolating functions  
\cite{wahba1990spline}.  We can write cubic spline interpolation in the form
\[
	\min_{f \in \mathcal H} \frac{1}{m} \sum_{ i = 1}^m {\mathcal L}(f(x_i),y_i), \qquad \mathcal H = \{ f(x) \text{ and } f'(x) \text{ continuous }\}.
\] 
Equivalently, the solution is  characterized by 
	\[
	\min_f \mathcal R^{curv}(f) = \int (f''(x))^2 dx, \quad \text{ subject to }  \{f(x_i) = y_i, i = 1,\dots m \}
	\]
\end{example} 

Regularized loss functionals such as \eqref{ELR} arise in mathematical approaches to image processing~\cite{aubert2006mathematical, sapiro2006geometric} as well as inverse problems in general.  The approach has a mature mathematical theory which includes stability, error analysis, numerical convergence, etc.  Mathematical Image processing tools have also been used for other important deep learning tasks, such as Image Segmentation.   In signal processing, the loss and the regularizer are designed to adapt to the signal and noise model.  So, for example, quadratic losses arise from Gaussian noise models.  Classical Tychonov regularization \cite{tikhonov2009solutions},   $\mathcal R^{Tych}(f) =\int_D  |\nabla f(x)|^2$ is compatible with smooth signals.  Total Variation regularization~\cite{rudin1992nonlinear}, $\mathcal R^{TV}(f) =\int_D  |\nabla f(x)|$  is compatible with images. 

\subsection{Curse of dimensionality}
Mathematical  approximation theory \cite{cheney1966introduction} allows us to prove  convergence of approximations $f^m \to f$ with rates which depend on the error of approximation  and on the typical distance between a sampled point and a given data point,  $h$.   For uniform sampling of the box $[0,1]^d$ with $m$ points, we have $h = m^{1/d}$.  When the convergence rate is a power of $h$,  this is an example of the \emph{curse of dimensionality}: the number of points required to achieve a given error grows exponentially in the dimension.   Since the dimension is large, and the number of points is in the millions, the bounds obtained are vacuous.   

There are situations in function approximation where convergence is exponentially fast.  For example, Fourier approximation methods can converge exponentially fast in $h$ when the function are smooth enough.   Next  we will see the how kernel methods can overcome the curse of dimensionality.  Later we will present a connection between kernel methods and Fourier regularization.

\section{Kernel methods}\label{sec:kernel}

The state of the art methods in machine learning until the mid 2010s were kernel methods \cite[Ch 6]{mohri2018foundations}, 
which are based on mapping the data  $x \in \DataSupp $ into a high dimensional feature space,  $\Phi : \DataSupp \to H$,  where $\Phi(x) = (\phi_1(x), \phi_2(x), \dots )$.  The hypothesis space consists of linear combinations of feature vectors,
\begin{equation}\label{kernel_linear}
 \mathcal H^{ker} = \left \{ f(x,w) \mid  f(x,w) = \sum_i w_i \phi_i(x) \right \} 
\end{equation}
The feature space is a reproducing kernel Hilbert space, $H$, which inherits an inner product from the mapping $\Phi$.   This allows costly inner products in $H$ to be replaced with a function evaluation
\begin{equation}\label{kxy}
	K(x,y) =  \Phi(x) \cdot \Phi(y) = \sum_i \phi_i(x).
	\cdot \phi_i(y)
\end{equation}
The regularized empirical loss functional is given by 
\begin{equation}\label{standard kernel opt}\tag{EL-K}
	\min_{f \in \mathcal H^{ker} } \frac{1}{m} \sum_{ i = 1}^m {\mathcal L}(f(x_i,w),y_i) + \frac \lambda 2 \|w\|^2_H
\end{equation}
For convex losses, \eqref{standard kernel opt} is a convex optimization in $w$.  For classification, the margin loss is used, and the optimization problem corresponds to quadratic programming.  In the case of quadratic losses, the optimization problem is quadratic, and the minimizer of \eqref{standard kernel opt} has the explicit form 
\begin{equation}
	f(x) = \sum_{i=1}^m w_i K(x,x_i), \qquad 	(M+\lambda I ) w  = y
\end{equation}
where the coefficients $c$ are given by the solution of the system of linear equations with $M_{ij} = K(x_i,x_j)$, and $I$ is the identity matrix.  Note that the regularization term has a stabilizing effect: the condition number of the system with $\lambda I$ improves with $\lambda > 0$.  Better conditioning of the linear system means that the optimal weights are less sensitive to changes in the data
 \[
 w^* = 	(M+\lambda I )^{-1} y
 \]
Algorithm to minimize \eqref{standard kernel opt} are designed to be written entirely in terms of inner products, allowing for high dimensional feature spaces. 

\section{Training and SGD}\label{sec:SGD}

\subsection{Optimization and variational problems}
Once we consider a fixed hypothesis class, \eqref{ELR} becomes \eqref{ELW}, which is a finite dimensional optimization problem for the parameters $w$.  Optimization problems are easier to study and faster to solve when they are convex~\cite{boyd2004convex}.   Support vector machines are affine functions of $w$ and $x$.  Kernel methods are affine functions of $w$. This makes the optimization for kernels methods a convex problem. 

We consider problems where the number of data points, $m$, the dimension of the data, $n$, and the number of parameters, $D$, are all large.  In this case the majority of optimization algorithms developed for smaller scale problems are impractical,  for the simple reason that it may not efficient to work with $m\times m$ matrices or take gradients of losses involving $m$ copies of $n$ dimensional data.  Stochastic gradient descent (SGD) has emerged as the most effective algorithm \cite{bottou2016optimization}: where previous algorithms tried to overcome large data by visiting each date point once, SGD uses many  more iterations, visiting data multiple times, making incremental progress towards the optimum at each iteration.   For this reason, the number of iterations of SGD is measured in \emph{epochs}, which corresponds to a unit of $m$ evaluations of $\nabla \loss(f(x_i),y_i)$.  

\subsection{Stochastic gradient descent}
 Evaluating the loss \eqref{ELH} on all $m$ data points can be costly. Define random minibatch $I \subset \{1,\dots, m\}$, and define the corresponding minibatch loss by
\[
L_I(w)=  \frac{1}{|I|} \sum_{i \in I} {\mathcal L}(f(x_i,w),y_i)
\]
Stochastic gradient descent corresponds to 
\begin{equation}
	\label{SGD_w}
	w^{k+1} = w^k - h_k \nabla_w L_{I_k}(w^k), \quad {I_k} \text{ random, $h_k$ learning rate} 
\end{equation}

\begin{example}[simple SDG example]
Let $x_i$ be i.i.d.\ samples from the uniform probability $\rho_1(x)$ for $x \in [0,1]^2$, the two dimensional unit square.  Consider estimating the mean using the following quadratic loss
\begin{equation}\label{ELM mean}\tag{EL-Q}
	\min_{f \in \mathcal H}   \frac 1 m \sum_{i=1}^m (x_i - f(x_i,w))^2
\end{equation}
where $\mathcal H = \{ f(x,w) = w \mid w \in \R^2\}$ is simply the set of two dimensional vectors representing the mean.   Empirical loss minimization of \eqref{ELM mean} corresponds to simply computing the sample mean, $w^* = w^*(S_m) = \sum x_i / m$.   The full dataset and a random minibatch are illustrated in Figure~\ref{fig:minibatch}.   As the value $w_k$ gets closer to the minimum, the error in the gradient coming from the minibatch increases, as illustrated in the Figure.  As a result a decreasing time step (learning rate) $h_k$ is used.  The schedule for $h_k$ is of order $1/k$, and the convergence rate for SGD, even in the strongly convex case, is also of order $1/k$.  
\end{example}

 \subsection{The convergence rate of SGD}
See \cite{bottou2016optimization} for a survey on results on SGD.  The following simple result was proved in \cite{MarianaSGD}.
	Suppose $f$ is $\mu$-strongly convex and $L$-smooth, with minimum at $w^*$. 	Let $q(w) = \|w-w^*\|^2$.
Write 
\[
\nabla_{mb} f(w ) = \nabla f(w)+ e,
\]
with $e$ is a mean zero random error term, with variance $\sigma^2$.
Consider the stochastic gradient descent iteration
\begin{align}\label{ASGD}\tag{SGD}
	w_{k+1} 
	= w_k - h_k \nabla_{mb} f(w_k),
\end{align}
with learning rate
\begin{equation}
	\tag{SLR}\label{LR}
	h_k = \frac{1}{\mu(k+q_0^{-1} \alpha_S^{-1}) },  
\end{equation}
\begin{theorem}\label{sgd_thm}
	Let $w_k,\ h_k$ be the sequence given by~\eqref{ASGD} \eqref{LR}.
	Then,
	\begin{equation}
	\mathbb{E}  \left [ {q_k \mid w_{k-1}}  \right ] \leq \frac{1}{\alpha_S k +q_0^{-1}}, \mbox{ for all $k\geq 0$.}
	\end{equation}
\end{theorem}

The convergence rate of SGD is slow: the error decreases on the order of $1/k$ for strongly convex problems.  This means that if it takes $1000$ iterations to reach an error of $\epsilon$, it may take ten times as many iterations to further decrease the error to $\epsilon /10$.    However, the the tradeoff of speed for memory is worth it: while the number of iterations to achieve a small error is large, the algorithm overcomes the memory bottleneck, which would make computing the full gradient of the loss function impractical.  

\subsection{Accelerated SGD}
In practice, better empirical results are achieved using the accelerated version of SGD.  This algorithm is the stochastic version of Nesterov's accelerated gradient descent \cite{nesterov2013introductory}.  See 
 \cite{goh2017why} for an exposition on Nesterov's method. 
Nesterov's method  can be interpreted as the discretization of a second order ODE~\cite{su2014differential}.  
In \cite{MaximeN2N} we show how the continuous time interpretation of Nesterov's method with stochastic gradients leads to accelerated convergence rates for Nesterov's SGD, using a Liapunov function analysis similar to the one described above. 

\begin{figure}
  \includegraphics[width = 3.5cm]{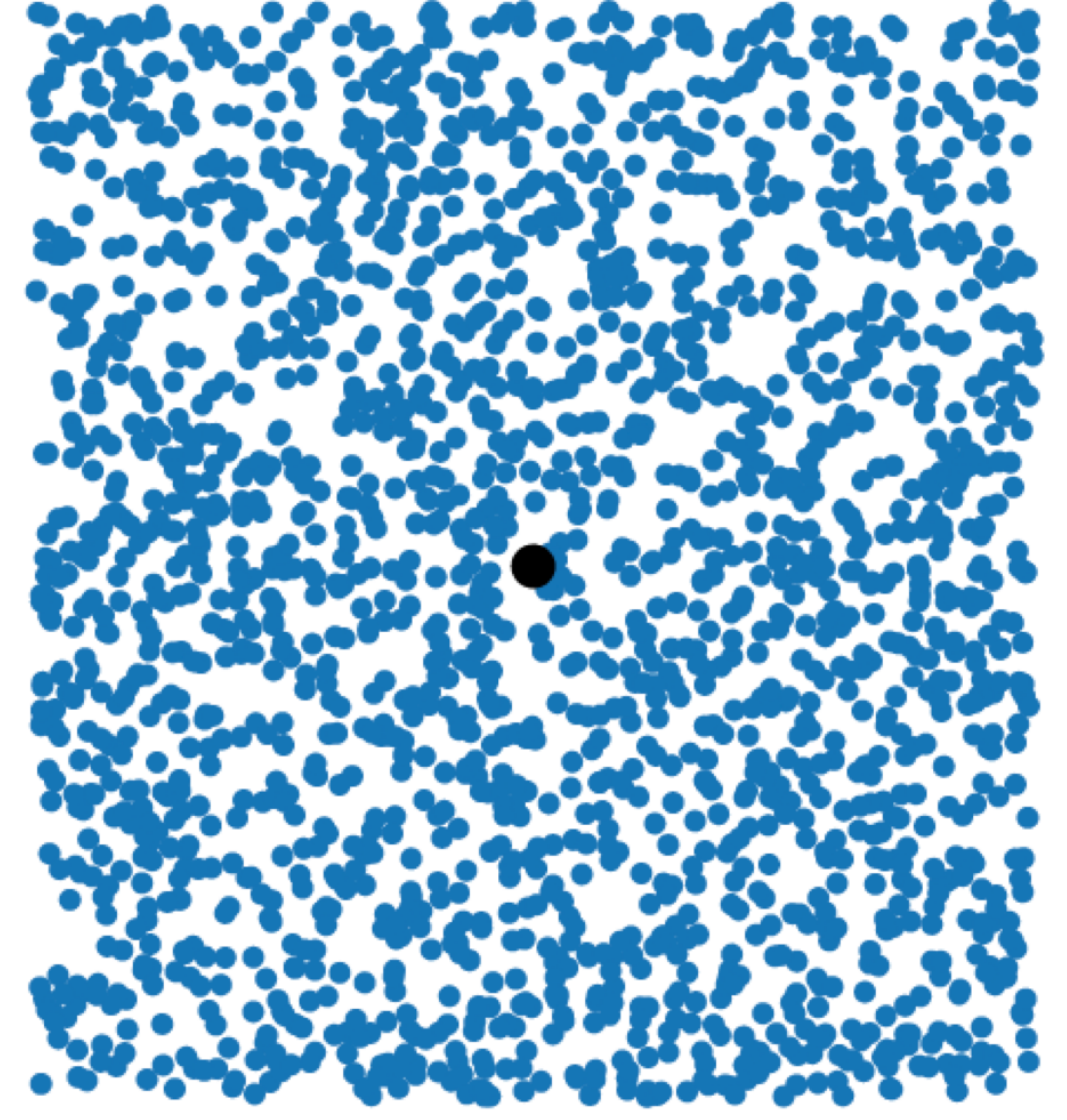}
    \includegraphics[width = 3.5cm]{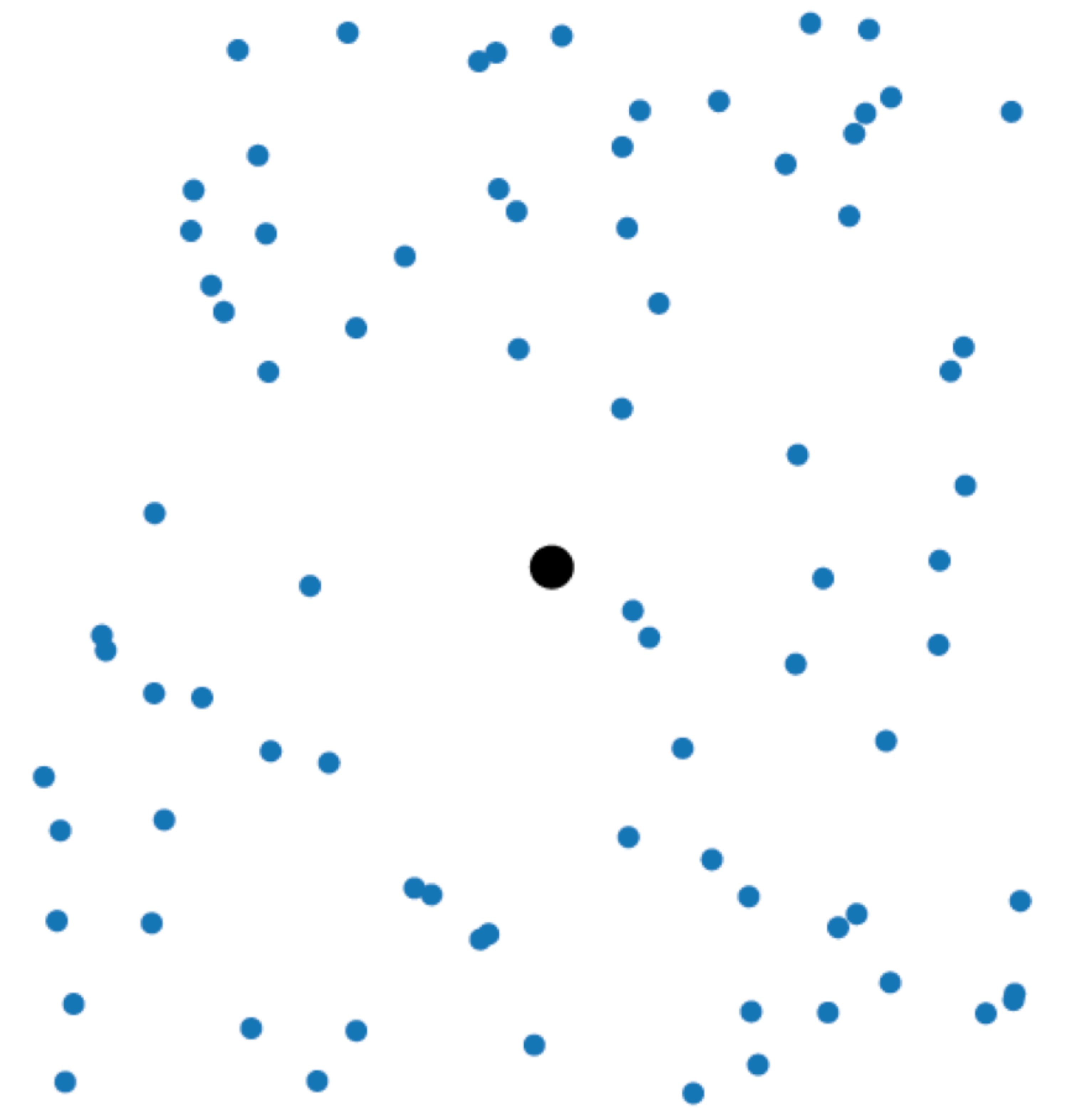}
\includegraphics[width = 5.1cm]{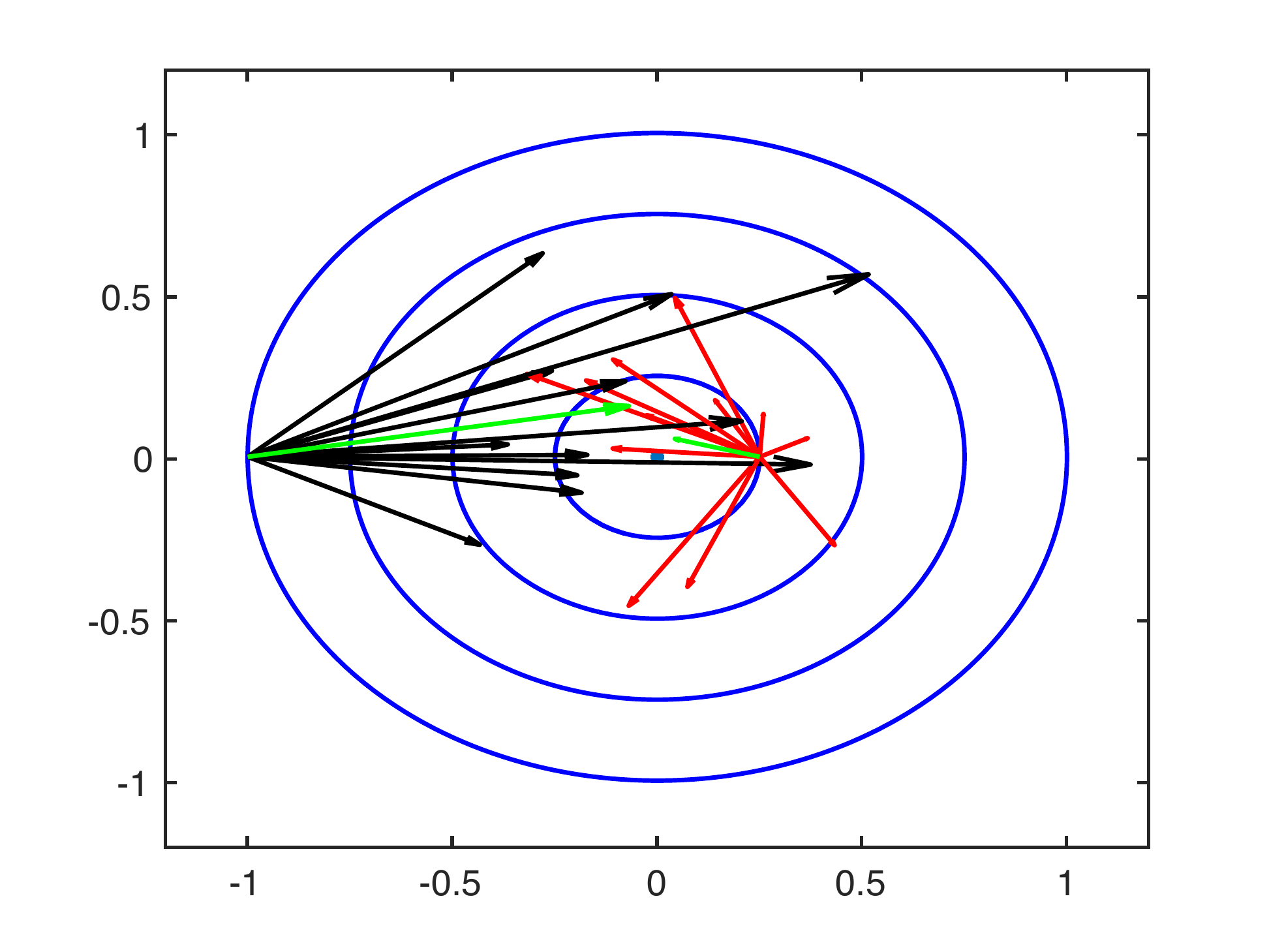}
  \caption{Full data set and a minibatch for  data sampled uniformly in the square.  Component gradients (black) and the minibatch gradient (green).  Closer to the minimum, the relative error in the minibatch gradient is larger.}
  \label{fig:minibatch}
\end{figure}

\section{Statistical learning theory}\label{sec:Stats}

\subsection{Concentration of measure}


Consider the experiment of flipping a possibly biased coin.  Let $X_k \in \{-1,1\}$ represent the outcomes of the coin toss.   After $m$ coin tosses, let $S_m = \frac 1 m \sum_{k=1}^m X_k$ be the sample mean.  The expected value of  $X$, $\mu = \EE[X]$ is the difference of the probabilities, $p_H - p_T$, which is zero when the coin is fair.   In practice, due the randomness, the sample mean will deviate from the mean, and we expect the deviation to decrease as $m$ increases. The Central Limit Theorem quantifies the deviation  $\sqrt{m}(S_m - \mu)$ converges to the normal distribution as $m \to \infty$, so that,
\begin{equation}\label{rough}
|S_m - \mu|   \approx   \frac 1 {\sqrt{m}}
\end{equation}
for $m$ large. Concentration of measure inequalities provide non-asymptotic bounds.  For example, Hoeffdings inequality~\cite[Appendix D]{mohri2018foundations} applied to random variables taking values in $[-1,1]$ gives
\[
P\left (|S_m - \mu | >  \epsilon \right )  \le 2 \exp \left (  - \frac{m \epsilon^2}{2} \right ) 
\] 
for any $\epsilon > 0$.   Setting $\delta = 2 \exp \left (  - {m \epsilon^2}/{2} \right )$ and solving for $\epsilon$  allows us to restate the result as 
\begin{equation}
	\label{berstien_log}
	|S_m - \mu | \leq \sqrt{  \frac{2\log(2/\delta) }{m}}
	\quad \text{with probability $1-\delta$, for any $\delta > 0$}.	
\end{equation}

\subsection{Statistical learning theory and generalization}
Statistical learning theory (see \cite[Chapter 3]{mohri2018foundations} and \cite[Chapter 4]{shalev-shwartz_ben-david_2014}) can be used to obtain dimension independent sample complexity bounds for  the expected loss (generalization) of a learning algorithm. These are bounds which depend on the number of samples, $m$ but not on the dimension of the underlying data, $n$.

 The \emph{hypothesis space complexity} approach restricts the  hypothesis space \eqref{Hdefn}  to limit the ability of functions to overfit by bounding the  the generalization gap  ${L}^{\textrm{gap}} [f]	 := \LD[f] - {L}_{S}[f]$.
The dependence of the generalization gap  on the learning algorithm is removed by considering the worst-case gap for functions in the hypothesis space 
\begin{equation}
 {L}^{\textrm{gap}}[f_{\mathcal A(S)}]  \leq \sup_{f \in \mathcal H}  {L}^{\textrm{gap}} [f]	 
\end{equation}
For example, \cite[Theorem 3.3]{mohri2018foundations} (which applies to the case of bounded loss function $0 \leq \loss \leq 1$),  states that for any $\delta > 0$
\begin{equation}\label{gen_2}
	 {L}^{\textrm{gap}}[f_{\mathcal A(S)}]  \leq \mathfrak{R}_m(\mathcal H) +  \sqrt{ \frac{\ln \frac 1 \delta}{2m }}, \quad \text{ with probability $\geq 1- \delta$}
\end{equation}
where $\mathfrak{R}(\mathcal H)$ is the Rademacher complexity of the
hypothesis space $\mathcal H$. 
Observe that \eqref{gen_2} has a similar form to \eqref{berstien_log}, with the additional term coming from the hypothesis space complexity.  Thus, restricting to a low-complexity hypothesis space reduces \emph{learning} bounds to \emph{sampling} bounds.   

The Rademacher complexity of $\mathcal H$ measures the probability that some function $f \in \mathcal H$  is able to fit samples of size $m$ with random labels~\cite[Defn 3.1 and 3.2]{mohri2018foundations}.  
\begin{example}
Consider the hypothesis space consisting of the set of axis aligned rectangles in the plane.  Set $x = (x_1, x_2), w = (w_1, w_2, w_3, w_4)$ and 
\[
f(x,w) = \begin{cases}
	1 & \text{ if  $w_1 \leq x_1 \leq w_2$, and $w_3 \leq x_2 \leq w_4$}\\
	0 & \text{ otherwise } 
\end{cases}
\]
and consider the hypothesis space given by 
\[
\mathcal H = \{ f(x,w) \mid w_1 \leq w_2, w_3 \leq w_4 \}
\]
Then $\mathfrak{R}(\mathcal H) \to 0$ as $m\to\infty$. 
\end{example}

\subsubsection*{Stability approach to generalization}
The \emph{stability} approach to generalization \cite[Chapter 13]{shalev-shwartz_ben-david_2014}  \cite{bousquet2002stability} considers perturbations of the data set $S$.  It measures how much changing a data point in $S$ leads to a change in $f_{\mathcal A(S)}$.  It also leads to estimates of the form \eqref{gen_2}, with the Rademacher complexity of the hypothesis space replaced by the uniform stability of the algorithm. For example, see \cite[Ch 13 and Ch 15]{shalev-shwartz_ben-david_2014} for ridge-regression and support vector machines.

\subsubsection*{Robustness  approach to generalization}
The \emph{robustness} approach  \cite{xu2012robustness} considers how much the loss value can vary with respect to the input space of $(x,y)$.   They define an algorithm $\mathcal A$ to be $(K, \epsilon(S))$-robust if the dataset can be partitioned into $K$ disjoint sets $\{C_i\}_{i=1}^K$ and there is a function $\epsilon(S) \geq 0$ such that for all $s \in S$
\begin{equation}
	s,z \in C_i \implies | \loss(f_{\mathcal A(S)}, s) - \loss(f_{\mathcal A(S)}, z) | \leq \epsilon(S)
\end{equation}
For $(K, \epsilon(S))$-robust algorithms, the result is
\begin{equation}\label{robust_gen}
	 {L}^{\textrm{gap}}[f_{\mathcal A(S)}] \leq \epsilon(S) + 2 M\sqrt{\frac{ {K \ln 2 + \ln \frac 1 \delta} }{2m}   }, \quad \text{ with probability $\geq 1- \delta$}
\end{equation}
So the result \eqref{robust_gen} trades robustness for Rademacher complexity, with the addition of a term measuring the number of sets in the partition. 

However, if we consider an $L$-Lipschitz function $f$, then the function is $\epsilon$ robust for a partition of the set by balls of radius $\epsilon/L$.  However the number of such balls depends on the dimension, $K = (1/\epsilon)^d$, so, in this case, the curse of dimensionality is still there, but absorbed into the constant $K$.

\section{Deep Neural Networks}\label{sec:Architecture}
 The deep learning hypothesis class is nonlinear and nonconvex in both $w$ and $x$.   This is different from support vector machines, which are affine in both variables, and kernel methods, which are affine in $w$.  The nonconvexity makes the analysis of the parametric problem \eqref{ELW} much more complicated.  On the other hand,  studying \eqref{ELR} in the nonparametric setting allows us to analyze the existence, uniqueness and stability of solutions, without the additional complicating details of the parameterization. 

\subsection{Network architecture}
The architecture of a deep neural network refers to the definition of the hypothesis class.  The function $f(x;w)$ is given by a composition of linear (affine) layers with a nonlinearity, $\sigma(t)$,  which is defined to act componentwise on vectors
\[
 \sigma(y_1,\dots, y_n) \equiv (\sigma(y_1), \dots, \sigma(y_n)).
\] 
In modern architectures, the nonlinearity is the rectified linear unit (ReLU), $\sigma(t) = \max(t,0)$.  Define the $i$th layer of the network using (consistently sized) rectangular matrix, $W_i$, and bias vector,  $b_i$ composed with the nonlinearity
\[
f^{(i)} = \sigma(W_i x + b_i).
\]
The neural network with $J$ layers is given by the composition of layers with consistent dimensions,
\begin{equation}
	f(x,w) = f_J \circ \dots \circ f_1, \quad
	w = (W_1, \dots, W_J) 
\end{equation}
where the parameter $w$ is the concatenation of the matrices in each layer $w = (W_1, \dots, W_J)$.  The network is deep when $J$ is large.  

Convolutional neural networks \cite[Chapter 9]{goodfellow2016deep} constrain the matrices to be sparse, leading to significant reduction in the total number of parameters.  The basic convolutional network has a sparsity structure of 9 points, corresponding to nearest neighbors on a two by two grid.  The weights are repeated at different grid points in such a way that the  matrix vector product, $Wx$, is equivalent to a convolution.

\subsection{Backpropagation}
Backpropagation is the terminology for symbolic differentiation of a neural network.  Using the parametric representation of the mode,  we can compute the full gradient vector 
\[
\nabla f = (\nabla_w f(x,w), \nabla_x f(x,w))
\]
using the chain rule.  The first component, $\nabla_w f(x,w)$ is used to optimize the weights for training.

\subsection{DNN Classification}
For classification problems with $K$ classes, the final layer of the neural network is $K$ dimensional and the classification is given by~\eqref{class}.   Here, we interpret the DNN classification loss as a smoothed version of the max loss \eqref{max_loss}.

The standard loss function for DNN classification is the composition of the Kullback-Liebler divergence with the softmax function
\[
\softmax(f) = \frac{1}{ \sum_{i=1}^K \exp(f_i)} (\exp(f_1), \dots, \exp(f_K)).
\]
The Kullback-Leibler divergence from information theory~\cite[Section 1.6]{bishop2006pattern} \cite[Section 3.13]{goodfellow2016deep} is defined on probability vectors, by 
\begin{equation}
\loss_{KL}(q,p) = - \sum_{i=1}^K p_i \log(q_i/p_i),
\end{equation}
When $y_i$ is the one hot vector (the standard basis vector $e_i$),  $(0, \dots, 1, \dots, 0)$ the composition results in 
\begin{equation}
\label{KLcomp}\tag{KL-SM}
\loss_{KL-SM}(f,k) = -\log(\softmax(f)_k) =  \log \left ( \sum \exp(f_i) \right) - f_k	
\end{equation}

The usual explanation for the loss is the probabilistic interpretation\footnote{see ``An introduction to entropy, cross entropy and KL divergence in machine learning,''``KL Divergence for Machine Learning," and ``Light on Math Machine Learning: Intuitive Guide to Understanding KL Divergence''} which applies when classifiers output, $p(x) = (p_1(x), \dots, p_K(x))$, an estimate of the probabilities of that the image $x$ is in the class $i$.  The probabilistic explanation is not valid, since (i) the KL-divergence is composed with a (hard coded, not learned) $\softmax$,  and (ii) there is no probabilistic interpretation of $\softmax(f)$.

Instead, we return to the maximum loss, \eqref{max_loss}, and consider a smooth approximation of $\max(v)$ given by $g^\epsilon(v)$ 
\begin{equation}
	\label{max_loss_smooth}
	\loss_{\max^\epsilon}(f,k) = g^\epsilon(f) - f_k
\end{equation}
For example, if we take
\begin{equation}
\label{smooth_max}
g^\epsilon(v) = \epsilon \log \left ( \sum_{i=1}^K \exp(f_i/\epsilon ) \right) 
\end{equation}
Then $g^\epsilon$ is smooth for $\epsilon>0$ and $g^\epsilon \to \max$ as $\epsilon \to 0$, since
\begin{equation}
	\max_i v_i \leq g^\epsilon(v) \leq \epsilon \log K  + g^\epsilon(v)
\end{equation}
which follows from 
\begin{align}
	\max_i v_i &= \log(  \exp( \max_i(v_i))) \leq  \log \left ( \sum_{i=1}^K \exp(f_i) \right)\\
	&
	 \leq \log ( K \exp( \max_i v)_i) = \max(v) + \log K.	
\end{align}
Using \eqref{max_loss_smooth} with  $\epsilon = 1$ in \eqref{smooth_max} we recover \eqref{KLcomp}.  We thus interpret the \eqref{KLcomp} loss as a smooth version of the classification loss \eqref{max_loss}. 

\begin{figure}
\centering
\includegraphics[width = 12cm]{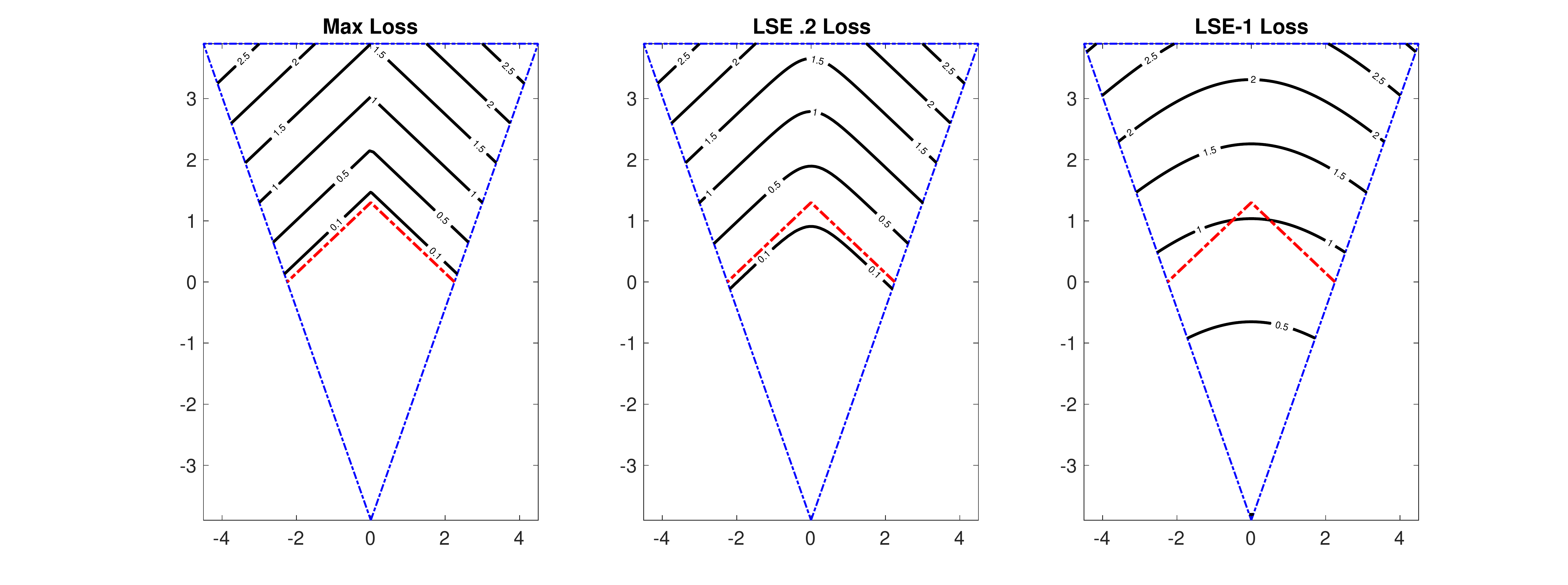}
  \caption{Classification Losses: level sets of different classification losses, 
 along with the classification boundary.  From left to right: max loss, LSE loss with $\epsilon = .2$ , LSE loss with $\epsilon = 1$.}
  \label{fig:losses}
\end{figure}

\subsection{Batch normalization}
Deep models require batch normalization \cite{ioffe2015batch}, which was introduced to solve the vanishing gradient problem, which arises when  training using the ReLU activation function.  In the case that
 \[
 W_j x_i < 0, \text{ for all } i = 1, \dots, m
 \]
 where $x < 0$ means that each component of the vector is negative, then $\sigma(W_j x)$ will always be zero, and $\nabla_w f(x) = 0$.  This is a problem even if the inequality above holds only for a particular mini-batch.  Then gradients are zero and  the network cannot update the weights.  Proper initialization can correct this problem, but for deep networks, batch normalization is still required for accuracy. 
  
  Batch normalization consists of adding a layer which has a shift and scaling to make the outputs of a layer mean zero and variance one.  Thus $f(x,w)$ now depends on the statistics of $S_m$, which is no longer consistent with the hypothesis class definition.  It is possible to enlarges the hypothesis class to include statistics of the data, but this makes the problem \eqref{ELH} more complicated.

\section{Adversarial attacks}

\begin{figure}
    \centering
        \includegraphics[width=.45\textwidth]{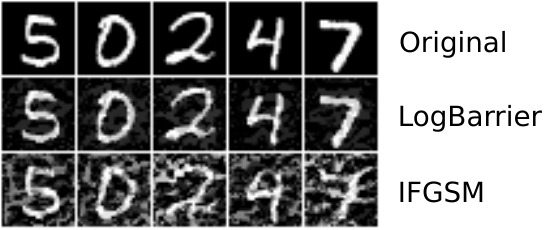}
        \includegraphics[width=.45\textwidth]{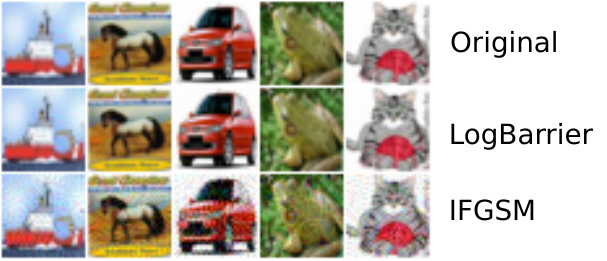}
    \caption{Adversarial images for  $\ell_\infty$ perturbations, generated by classification (LogBarrier) and loss (IFGSM) adversarial attacks, compared to the
      original clean image, for MNIST (left) and CIFAR10 (right). Where IFGSM has
    difficulties finding low distortion adversarial images, the classification attack succeeds.}\label{fig:pics}
\end{figure}

Robustness of $f(x,w)$ refers to the lack of sensitivity of the model to small changes in the data.  The existence of adversarial examples \cite{goodfellow2014explaining} indicates that the models are not robust.  
The objective of adversarial attacks is to find the minimum norm vector which leads to misclassification, 
\begin{equation}\label{class_attack}
	\min \left\{ \|v\| \mid  C(f(x+v))  \not=C((f(x))  \right\}.
\end{equation}
The classification attack \eqref{class_attack} corresponds to a global, non-differentiable optimization problem.  A more tractable problem is given by attacking a loss function. The choice of loss used for adversarial attacks can be a different one from \eqref{KLcomp} loss used for training.  The max-loss \eqref{max_loss}, as well as a smoothed version, was used for adversarial attacks on deep neural networks in~\cite{carlini_towards_2016}.  Adversarial attacks are illustrated in Figure~\ref{fig:pics}.

Write $\ell(x) = \loss(f(x),y(x))$ for the loss of the model.  For a given adversarial distance, $\lambda$, the optimal loss attack on the image vector, $x$, is defined as
\begin{equation}
\label{ideal_attack}
\max_{ \|v\| \leq \lambda} \ell(x + v).
\end{equation}
The solution, $x+v$,  is the perturbed image vector within a given distance of $x$ which maximally increases the loss.  The Fast Gradient Sign Method (FGSM) \cite{goodfellow2014explaining} arises when
attacks are measured in the $\infty$-norm.  It corresponds to a one step attack in
the direction $v$ given by the signed gradient vector
\begin{align} \label{eq:l1_dual_vector}	
v_i = \frac{ \nabla \ell(x)_i }{| \nabla \ell(x)_i|}.
\end{align}
The attack direction \eqref{eq:l1_dual_vector} arises from linearization of the objective in \eqref{ideal_attack}, which leads to 
\begin{equation}\label{norm_infty_dual}
	\max_{ \|v\|_\infty \leq 1}  v \cdot \nabla \ell(x) 
\end{equation} 
By inspection, the minimizer is the signed gradient vector,
\eqref{eq:l1_dual_vector}, and the optimal value 
is $\| \nabla \ell(x)\|_1$.  When the  $2$-norm is used in
\eqref{norm_infty_dual}, the optimal value is $\| \nabla
\ell(x)\|_2$ and the optimal direction is the normalized gradient
\begin{align} \label{eq:l2_dual_vector}
v =   \frac{ \nabla \ell(x) }{\| \nabla \ell(x) \|_2}.
\end{align}
More generally, when a generic norm is used in \eqref{norm_infty_dual}, the
maximum of the linearized objective defines the dual norm~\cite[A.1.6]{boyd2004convex}.

\subsection{Classification attacks}
In \cite{AramAttack}  we implemented to the barrier method from constrained optimization \cite{nocedal} to perform the classification attack \eqref{class_attack}.  While attacks vectors are normally small enough to be invisible, for some images, gradient based attacks are visible.  The barrier attack method generally performs as well as the best gradient based attacks, and on the most challenging examples results in smaller attack vectors,  see Figure~\ref{fig:LogBarrier}.
\begin{figure}
  \includegraphics[height = 4.0cm]{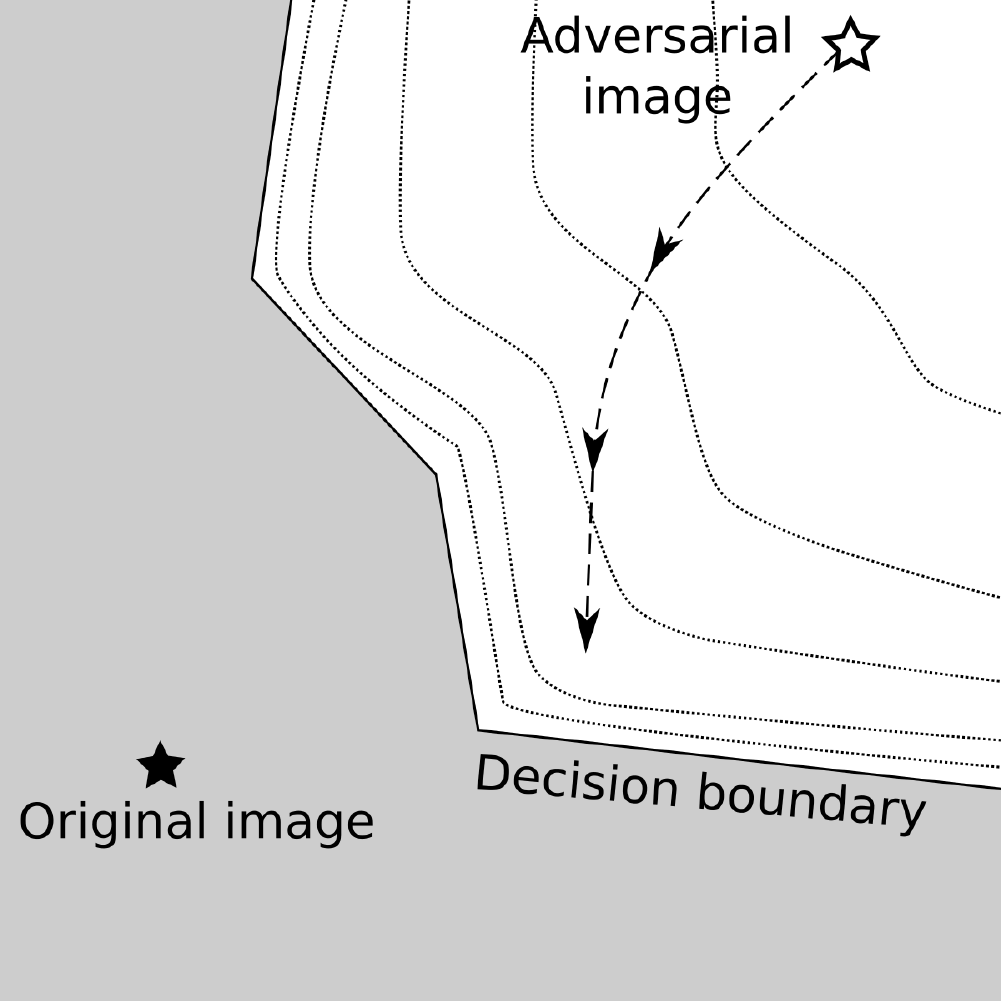}
    \caption{Illustration of an adversarial attack on the classification boundary.}
  \label{fig:LogBarrier}
\end{figure}

\section{Regularization in DNN}
\subsection{Statistical Learning Theory approach to generalization in Machine Learning}
In practice, neural networks can perfectly fit random labelings of the data  \cite{zhang2016understanding}.  
Thus, the Rademacher complexity of the neural network hypothesis class is one which means that generalization of neural networks can not be established using the hypothesis space complexity approach.    The work of \cite{zhang2016understanding} did not use data augmentation, which is typically used when training networks to achieve better generalization.   In the sequel we will study when data augmentation can be interpreted as a form of regularization, with the hope that it can be used to better understand generalization of deep neural networks. 

\subsection{Regularization in DNN practice}
Various forms of \emph{algorithmic regularization} are implemented to improve generalization~\cite[Chapter 7]{goodfellow2016deep}.   Examples follow, which will then be interpreted as variational regularization. 

Early work on shallow neural networks implemented Tychonoff gradient regularization of the form  \eqref{ELR} directly: \cite{drucker1992improving} showed that it improved generalization error.  In ~\cite{bishop1995training}, it was shown that regularization can be achieved by adding random noise to the data $x$, which avoids the additional computational cost of explicit gradient regularization. 

Data augmentation~\cite{lecun1998gradient} improves generalization using simple image transformations such as small rotations, cropping, or intensity or contrast adjustment.  

Dropout~\cite{srivastava2014dropout} consists of randomly, with small probability, changing some model weights to zero.   Dropout  is another form of regularization, but this time the transformation is a perturbation of the model weights. A Bayesian interpretation of dropout \cite{gal2016dropout} was later retracted \cite{hron2017variational}.
  
 Cutout~\cite{cutout}  consists of randomly sending a smaller square of pixels in the image to zero (or one), and recently outperformed dropout on benchmarks.  

Mixup~\cite{mixup} consists of taking convex combinations of data points $x_i,x_j$ and the corresponding labels $y_i,y_j$. 
 
Gaussian noise averaging has recently reappeared in deep neural networks \cite{lecuyer2018certified, liu2018towards, 1809.03113, cohen2019certified} as a method to certify a network to be robust to adversarial perturbations.   The averaged network adds random noise with variance $\sigma^2$ and chooses the most likely classifier,
\[
C_{smooth}(x) = \argmax C(x + \eta),\qquad \eta = N(0,\sigma^2),
\]
which  requires many evaluations of the  network.

\section{PDE regularization interpretation}
In this section we demonstrate how  algorithmic or implicit regularization approaches to deep learning can be made explicit.   The first point we make is that  we can rewrite the kernel loss as in the form of \eqref{ELR}.   This suggests that generalization of  kernel methods could also be approached from the point of view of regularization.  The regularization functional \eqref{ELR} is non-parametric, which brings it closer to the approach we take for highly expressive  deep neural networks.  Next we show how various forms of  implicit regularization used in DNNs can be made explicit. 

\subsection{Kernels and regularization}
Observe that \eqref{standard kernel opt} corresponds to \eqref{ELW} with a regularization term involving the Hilbert space norm of the weights.  
Mathematically,  \eqref{standard kernel opt} is a strange object, because it involves functions, weights and the Hilbert space norm.  
We will rewrite in the form of a regularized functional \eqref{ELR}.   Regularization interpretation of kernels is discussed in \cite{girosi1995regularization, smola1998regularization, wahba1990spline}.

Consider the case where $K(x_1,x_2) = G(x_1-x_2)$ where $G$ is real and symmetric, and the Fourier transform $\hat G(y)$ is a symmetric, positive function that goes to zero as $y \to \infty$.  Then it can be shown that~\eqref{standard kernel opt} corresponds to \eqref{ELR} with 
\begin{equation}
	\label{Reg_Fourier}
	\mathcal R^{Ker}(f) = \int_{\Rn} \frac{ \|\hat f(y)\|^2 }{\hat G(y)} \, dy,
\end{equation}
where $\hat f$ is the Fourier transform of $f$.  Refer to~\cite{girosi1995regularization}.

\begin{example}\label{ex:Gaussian}
 See \cite{smola1998regularization} for details.
	The Gaussian kernel corresponds to 
\begin{equation}
	G(x)  = \exp(-\|x\|^2/2), \qquad  \hat G(y) = C \exp(-\|y\|^2/2).
\end{equation}
 In this case, the regularization is given by
\[
\mathcal R^{Ker}(f) =  \sum_{n=0}^\infty \frac{ 1}{2^n n!}\| (\nabla)^n f \|_{L^2}^2.
\]
%
\end{example}
Thus we see that kernel methods can be interpreted as regularized functional \eqref{ELR} with Fourier regularization.

\subsection{SGD and regularization}
Some authors seek to explain generalization using the properties of the SGD training algorithm, for example~\cite{hardt2015train}.  We consider two examples  in the overparameterized setting, where the number of parameters can be greater than the number of data points.

The first example shows that without regularization, SDG can train to zero loss and fail to generalize.  The second example shows that implicit regularization by smoothing the function class can lead to generalization.  The point of these  examples  to show that SGD does not regularize without additional implicit regularization.

\begin{example}
Consider the dataset $S_m$, defined in~\eqref{dataset},  with $x_i$ i.i.d.\ samples from the uniformly probability $\rho_1(x)$ on  $[0,1]^2$, the two dimensional unit square.  Set 
\[
y(x)  = \rho_1(x) 
= \begin{cases}
	1, & x \in [0,1]^2\\
	0, & \text{ otherwise} 
\end{cases}
\]
Consider the quadratic loss \eqref{ELM mean} with the overparameterized hypothesis space 
\begin{equation}\label{H over parameterized}
	 \mathcal H^\delta = \left \{ f(x,w) = \sum_{i=1}^m w_i \delta_{x_i}(x)  \right \}
\end{equation}
where   $\delta_{x_i}(x) = 1$ if $x=x_i$, and zero otherwise. 	
\end{example}
In this case, it is clear that (i) setting $w^*_i = 1$ makes the empirical loss zero, and  (ii) the generalization error is large, since on a new data point, $x$, $f(x,w^*) = 0  \not = y(x)$.  

However, if we replace $\mathcal H^\delta$ with a smooth hypothesis class
\begin{equation}\label{Hsmooth}
	 \mathcal H^G = \left \{ f(x,w) = \sum_{i=1}^m w_i G(x_i-x)  \right \}
\end{equation}
where $G$ is a smooth approximation of the delta function (for example a Gaussian kernel as in Example~\ref{ex:Gaussian}), then it is possible to generalize better.   Adding a parameter representing the width of the kernel allows us to tune for the optimal width, leading to better regularization for a given sample size. 

Learning guarantees for classification using kernels are usually interpreted using the notion of margin \cite[Ch 6]{mohri2018foundations}.  However, this simple example illustrated that the interpretation \eqref{ELR} can also be used, with the kernel width corresponding to the regularization parameter $\lambda$.  

\subsection{Image transformation and implicit regularization}
Consider an abstract data transformation, 
\begin{equation}
	\label{data_augment}
x \mapsto T(x)
\end{equation}
which transforms the image $x$.  This could be data augmentation, random cutout, adding random gaussian noise to an image, or an adversarial perturbation.    The data transformation replaces \eqref{ELH} with the data augmented version
\begin{equation}\label{EL-A}\tag{EL-A}
	\min_{f\in \mathcal H} \frac{1}{m} \sum_{ i = 1}^m \loss (f(T(x_i)),y_i) 
\end{equation}
which we can rewrite as  
\[
\min_{f\in \mathcal H} L_{S_m}(f)+ \mathcal R^{T}(f(x_i))
\]
 where 
\begin{equation}\label{RT}\tag{RT}
	\mathcal R^T(f(x_i)) = \loss (f(T(x_i)),y_i) -  \loss (f(x_i),y_i)
\end{equation}
Here the regularization is implicit, and the strength of the regularization is controlled by making the transformation $T$ closer to the identity.

  \subsection{Data augmentation}
Here show how adding noise leads to regularization, following~ \cite{bishop1995training}.

\begin{lemma}
	The augmented loss problem \eqref{EL-A} with quadratic loss  and additive noise \begin{equation}\label{ExV}
	T(x) = x + v, \quad  \text{$v$ random}, \quad \EE(v) = 0, \quad \EE( v_i v_j ) = \sigma^2
\end{equation}
 is equivalent to the regularized loss problem \eqref{ELR} with
 \begin{equation}\label{Noise_Tychonoff}
		\mathcal R^{noise}(f) =  \frac{\sigma^2}{m} \sum_{ i = 1}^m \left ( f_x^2 +(f-y) f_{xx} + \frac{\sigma^2}{4} f_{xx}^2  \right )
	\end{equation}
\end{lemma}

\begin{proof}
For clarity we treat the function as one dimensional. A similar calculation can be done in the higher dimensional case. Apply the Taylor expansion 
\[
f(x+v) = f(x) + vf_x + \frac 1 2 v^2 f_{xx} + O(v^3)
\] 
to the quadratic loss $\loss(f,y) = (f(x + v) - y)^2$.  	Keeping only the lowest order terms, we have 
	\begin{align}
		(f(x + v) - y)^2 &= (f(x) - y)^2 + 2(f_x v + \frac 1 2 v^2 f_{xx})(f(x)-y) + (f_x v + \frac 1 2 v^2 f_{xx})^2 
	\end{align}
Taking expectations and applying \eqref{ExV} to drop the terms with odd powers of $v$ gives \eqref{Noise_Tychonoff}.
\end{proof}

\begin{remark}
	Note that this regularization is   empirical, unlike the previous ones. So the regularization may not lead to a well posed problem.  A more technical argument could lead to a well-defined functional, where the density depends on the sampled density convolved with a Gaussian.
\end{remark}

\subsection{Adversarial training}
In \cite{FinlayCalder} it was shown that adversarial training, 
\[
T(x) = x + \lambda v
\]
with attack vector  $v$ given by \eqref{eq:l1_dual_vector}	or by \eqref{eq:l2_dual_vector} 
can be  interpreted as Total Variation regularization of the loss,  
\begin{equation}\label{Noise_TV}
		\mathcal R^{AT}(f) = \lambda \|\loss'(f) \nabla f(x) \|_*.
\end{equation}

A different scaling was considered in \cite{FinlayScaleable}, which corresponds to adversarial training with $
T(x) = x + \lambda \nabla \loss(x).$ The corresponding regularization is Tychonoff regularization of the loss, 
\begin{equation}\label{Noise_T2}
		\mathcal R^{Tyc}(f) = \lambda  \| \loss'(f) \nabla f(x) \|^2_2
\end{equation}
which was implemented using finite differences rather than backpropagation. 
Double backpropagation is computationally expensive because the loss function depends on $\nabla_x f(x,w)$, training requires a mixed derivative $\nabla_x (\| \nabla_x f \|^2)$.  For deep networks, the mixed derivatives leads to very large matrix multiplication.  Scaling the regularization term by $\lambda$ and allowed for robustness comparable to many steps of adversarial training, at the cost of a single step.

\section{Conclusions}
We studied  the regularization approach to generalization and robustness in deep learning for image classification.   Deep learning models lack the theoretical foundations of traditional machine learning methods, in particular dimension independent sample complexity bounds of the form \eqref{gen_2}.  We showed that kernel methods have a regularization interpretation, which suggests that the same bounds can be obtained by considering the regularized functional \eqref{ELR}.   

The deep learning hypothesis space is too expressive for hypothesis space complexity bounds such as \eqref{gen_2} to be obtained.   However, that argument ignores the data augmentation, which is known to be a form of regularization.  We showed that many other modern data augmentation methods, including adversarial training, can also be interpreted as regularization.  The regularized loss \eqref{RT} may be more amenable to analysis than  \eqref{EL-A}, provided the regularizer is enough to make the problem mathematically well-posed.    Examples coming from data augmentation and adversarial training to show that \eqref{RT} can be represented as an explicit PDE regularized model.

The regularization interpretation makes a link between traditional machine learning methods and deep learning models, but requires nontraditional interpretations in both settings.  Proving generalization for nonparametric models with Fourier regularization \eqref{Reg_Fourier} could be  first step towards generalization bounds for regularized neural networks.   While the regularization interpretation \eqref{RT} is empirical, data augmentation rich enough (such as adding Gaussian noise)  to make the empirical measure supported on the full data distribution could lead to a well-posed global regularization.

\bibliographystyle{alpha}

\newcommand{\etalchar}[1]{$^{#1}$}

\end{document}